\documentclass[11pt,a4wide]{article}
\pdfoutput=1
\usepackage{amssymb,a4wide}
\usepackage{amsmath}
\usepackage{latexsym,url}
\usepackage{ascii}
\usepackage{color}
\usepackage{xcolor}
\usepackage{xspace}
\usepackage{ntheorem}
\usepackage{dsfont}
\usepackage{extarrows}
\usepackage{authblk}

\usepackage{tikz}
\usetikzlibrary{automata}
\usepackage{pgf,pgfarrows,pgfnodes}
\usepackage{multicol,xspace}
\usepackage {pgflibraryshapes}
\usepackage{eurosym}
\usetikzlibrary{arrows,backgrounds,positioning,fit,calc,shadows}
\newcommand{\avg}{\operatornamewithlimits{avg}}

\newtheorem{theorem}{Theorem}

\newtheorem{definition}[theorem]{Definition}

\newtheorem{proposition}[theorem]{Proposition}

\PassOptionsToPackage{dvipsnames}{xcolor}

\usepackage{pstcol}
\usepackage{pstricks}
\usepackage{graphicx}
\usepackage{pst-light3d}
\usepackage{color}
\usepackage{rotating}

\usepackage{times}
\usepackage{etex}
\usepackage{array}
\usepackage{latexsym}
\usepackage{mathrsfs}
\usepackage{graphicx}
\usepackage{amssymb}
\usepackage{jneurosci}
\usepackage{enumerate}
\usepackage[linesnumbered,ruled,vlined]{algorithm2e}
\usepackage{color}
\usepackage{multirow}
\usepackage{amsmath}
\usepackage{cases}
\usepackage{esvect}
\usepackage[T1]{fontenc}
\usepackage{algpseudocode}
\usepackage{ntheorem}
\usepackage{tikz}
\usepackage{pgf}
\usepackage{dsfont}
\usepackage{lmodern}
\usepackage{multirow}
\usepackage{times}
\usepackage{latexsym}

\usepackage{booktabs}
\usepackage{balance}

\newcommand{\pnote}[1]{}

\newcommand{\bgc}[1][]{
	\ifthenelse{\equal{#1}{}}{G}{G^{#1}}
}

\newlength{\wordlength}

\renewcommand{\P}{\ensuremath{\mathsf PTIME}}

\newtheorem{example}[theorem]{Example}
\newenvironment{proof}[1][Proof Sketch]{\begin{trivlist}%
	\item[\hskip \labelsep {\it #1:}]}{%
\end{trivlist}}%

\newenvironment{psketch}{\begin{proof}[Proof sketch]}{%
\end{proof}}
\renewcommand{\P}{\ensuremath{\mathsf P}}

\begin{document}

%\author{\#378}

% vigour! the text has to be interesting, not boring.

\title{Computing rational decisions \\ in extensive games with limited foresight} %\\ (draft, please do not quote)}

\author{Paolo Turrini}

\affil{\small{Department of Computing, Imperial College London}}

\maketitle

%%IJCAI guidelines at the bottom

\begin{abstract}

We introduce a class of extensive form games where players might not be able to foresee the possible consequences of their decisions and form a model of their opponents which they exploit to achieve a more profitable outcome. We improve upon existing models of games with limited foresight, endowing players with the ability of higher-order reasoning and proposing a novel solution concept to address intuitions coming from real game play. We analyse the resulting equilibria, devising an effective procedure to compute them.

\end{abstract}

\section{Introduction}

\noindent While game theory is a predominant paradigm in Artificial Intelligence, the tools it provides to analyse real game play still abstract away from many essential features. One of them is the fact that in a wide range of extensive games of perfect information (e.g., Chess), humans (and supercomputers) are generally not able to fully assess the consequence of their own decisions and need to resort to a judgment call before making a move. As acclaimed game theorist Ariel Rubinstein puts it, "modeling games with limited foresight remains a great challenge" and the game-theoretic frameworks developed thus far "fall short of capturing the spirit of limited-foresight reasoning" \cite[p.134]{Rubinstein-Bounded}.

On the contrary, the AI approach to game-playing builds upon the assumption that complex extensive games like Chess or Go are {\em theoretically} games of perfect information, but this is only marginally relevant for practical purposes, and the backwards induction solution is of little help in predicting how such games are actually played in practice - a point also raised in Joseph Halpern's AAMAS 2011 invited talk "Beyond Nash Equilibrium: Solution Concepts for the 21st Century" \cite{halpern}.  Decisions are instead taken using heuristic search (e.g., monte-carlo tree search) under various constraints, such as time or memory \cite{RussellW91} \cite{Russel-Norvig}.

{\bf The problem}. Search methods are a framework to handle limited foresight and are widely used for decision-making in real game-play, but a game-theoretic analysis of their equilibrium behaviour is still missing. In particular, we lack the tools to analyse what will happen in complex extensive games of perfect information where players are not able to resort to backwards induction reasoning but to possibly faulty and incomplete heuristic. What is more, the enormous effort to construct players with "opponent modelling" in the AI community (e.g., \cite{schadd2007opponent} \cite{Donkers}) still lacks solid game-theoretic foundations.

{\bf Our contribution}. We introduce games in which players might not be able to foresee the consequence of their strategic decisions all the way up to the terminal nodes and evaluate intermediate nodes according to a concrete heuristic search method. On top of that they can reason about other players' limited foresight and evaluation criteria: they are endowed with higher-order beliefs about what their opponents can perceive of the game and how they evaluate it, beliefs about what their opponents believe the others can see and how the evaluate it, and so forth.
To analyse these games, we propose a new solution concept which combines higher-order reasoning about players' limited foresight and evaluation criteria. {\em The guiding principle for players' behaviour is that each of them chooses a strategy in the game she sees that is a best response to the belief about what the other players can see and how they evaluate it}. We show constructively (Algorithms 1-4) that this solution concept always exists (Theorem 1) and is a strict generalization of other known ones, e.g., backwards induction. As we will observe, the unbounded chain of beliefs underlying our rationality constraints  can be finitely represented and - rather surprisingly - effectively resolved (Proposition \ref{prop:complexity}).% There is a logical background to this in fixed-point logics of belief, that we will also briefly touch upon.

{\bf Related literature}. In  recent years an innovative tradition has emerged in game theory, aiming at capturing situations in which players are unaware of parts of the game they are playing and might even think to be playing a different game from the real one.  Halpern and  R\^{e}go \cite{HalpernR06}, for instance, study models of unawareness of elements of the game played (e.g., other players). Yossi Feinberg \cite{Feinberg} approaches similar problems from a syntactic perspective. Simultaneously, the interplay between belief and awareness in interactive situations is analysed in a series of papers by Heifetz, Meier and Schipper \cite{HeifetzMS06}, \cite{HeifetzMS13}, \cite{HeifetzMS13a}.  %\cite{vanDitmarsch:2012} considered logics modeling all variety of static and dynamic aspects of awareness and knowledge.

It should be noted that even though all these frameworks abstractly allow to talk about unawareness of {\em some} terminal histories in a game, none of them comes equipped with a solution concept capturing limited foresight reasoning.

A framework that comes closest, perhaps, to this is {\em Games with Short Sight} \cite{GrossiT12:Sight}, a well-behaved collection of games with awareness, in which players of an extensive game make choices without knowing the consequences of their actions and base their decisions on a (possibly incorrect) evaluation of intermediate game positions.

 Games with Short Sight (GSSs) have been studied in relation with a solution concept
called {\em sight-compatible backwards induction}: as players might not be able
to calculate all possible moves up to the terminal nodes, they play rationally
 in a {\em local} sense, executing moves that are backwards induction moves in their
own sight, therefore {\em safely} assuming their opponents see as much of the game as they do.

However, sight-compatible backwards induction precludes any sort of {\em opponent modelling}, as players are not allowed to have a non-trivial belief about what their opponents perceive.
Thus, the tools developed in  \cite{GrossiT12:Sight} to analyse GSSs only allow players to play approximately or inaccurately, they don't allow players to exploit their opponents' believed weaknesses.
Besides, GSSs employ heuristics which are not grounded in practical game-play. Essentially, players come equipped with a preference relation over all histories of the game.

We will avoid strong rationality requirements of this kind, by introducing a significantly higher level of complexity in players' reasoning - notably their ability of forming an "opponent model" -  which, it turns out, still remains computationally manageable. Also players' preference relations will not be taken as given, but derived from concrete search methods. 

An important research line in AI that has similarities with our approach is interactive POMDPs \cite{IPOMDPs1}, which is able to incorporate higher-order epistemic notions in multi-agent decision making, with focus on learning and value/policy iteration. These graph-like models are generally highly complex - in fact the whole approach is known to suffer from severe complexity problems when it comes to equilibrium analysis and approximation methods have been devised to (partially) address them \cite{IPOMDPs2},\cite{IPOMDPs3}. 
Instead, we present a full-blown game-theoretic model of limited foresight that allows for higher-order epistemic notions and yet keeps equilibrium computation within polynomial time.

{{\bf Paper Structure}.} Section "Games with limited foresight" recalls useful formal notation and definitions from the literature upon which we build and introduces the mathematical structures we will be working on, Monte-Carlo Tree Games. Section "Rational beliefs and limited foresight" studies the higher-order extension thereof, Epistemic Monte-Carlo Tree Games. Specifically, we go on and define a new solution concept which takes this higher-order dimension into account and we then show the existence of the new equilibria through an efficient (\P-TIME) algorithm. Section "Conclusion and potential developments" summarises our findings and hints at new research avenues opening up in our framework.

\section{Games with limited foresight}\label{sec:Preliminaries}

We start out with the definition of extensive games, on top of which we build the models of limited foresight. 

\paragraph{Extensive Games}

An {\bf extensive game form}  \cite{Osborne1994} is a tuple  $$(N,H,t,{\rm\Sigma}_{i},o)$$ where [1] $N$ is a finite non-empty set of {\bf players}. [2] $H$ is a non-empty prefix-closed set of sequences, called {\bf histories}, drawn from a set $A$ of {\bf actions}. %, that satisfies the following three properties:
%   (i) the empty sequence $\emptyset$ is a member of $H$;
  % (ii) if $(a^k)_{k=1,...,K} \in H$ and $L<K$ then $a^k_{k=1,...,L} \in H$;
   %(iii) if an infinite sequence $(a^k)_{k=1,\ldots}$ is such that $(a^k)_{k=1,...,L} \in H$ for every positive integer $L$ then$(a^k)_{k=1,\ldots} \in H$.
A history ($a^k)_{k=1,\ldots,K} \in H$ is called {\bf terminal} history if it is infinite or if there is no $a^{K+1}$ such that $(a^k)_{k=1,...,K+1}\in H$. The set of terminal histories is denoted $Z$. A history $h$ is instead called {\bf quasi-terminal} if for each $a\in A$, if $(h,a)\in H$, then $(h,a)$ is terminal. If $h\in H$ is a prefix (resp., strict prefix) of $h'\in H$ we write $h \unlhd h'$ (resp., $h \lhd h'$). With $A_h=\{a \in A \mid (h,a)\in H\}$ we denote the set of actions following the history $h$. %The composition of a history $h$ with a set of histories $H'\subseteq H$ is denoted $h\circ H^{\prime} = \{(h,h') \mid h' \in H'\}$ while
The restriction of $H'\subseteq H$ to $h\in H$, i.e., $\{(h,h^{\prime})\in H \mid (h,h') \in H'\}$ is denoted $H'|_h$. [3]$t: H\setminus Z \rightarrow N$ is a {\bf turn function}, which assigns a player to each non-terminal history, i.e., the player who {\em moves} at that history.  [4] ${\rm\Sigma}_i$ is a non-empty set of {\bf strategies}. A strategy of player $i$ is function $\sigma_i: \{h \in {H\backslash Z} \mid t(h)=i\} \rightarrow A_h$, which assigns an action in $A_h$ to each non-terminal history for which $t(h)=i$. [5] ${o}$ is the {\bf outcome function}. For each strategy profile ${\rm\sigma}=\prod_{i\in N}(\sigma_i)$, the outcome ${o}({\rm\sigma})$ of $\rm\Sigma$ is the terminal history that results when each player $i$ follows the precepts of $\sigma_i$. %That is, $\mathcal{o}(\rm\sigma)$ is the terminal history $(a^1,\ldots,a^K) \in Z$ such that $0\leq k < K$, we have $\sigma_{t(a^1,\ldots,a^k)}=a^{k+1}$.

An {\bf extensive game} is a tuple $\mathcal{E}=(\mathcal{G},\{u_i\}_{i\in N})$, where $\mathcal{G}$ is an extensive game form, and $u_i: Z \to \mathbb{R}$ is a utility function for each player $i$, mapping terminal histories to reals. We denote $\succeq_i \subseteq Z\times Z$ the induced total preorder over $Z$ and ${\bf BI}(\mathcal{E})$ the set of backwards induction histories of extensive game $\mathcal{E}$, computed with the standard procedure \cite[Proposition 99.2]{Osborne1994}.

\paragraph{Sight Functions and Forked Extensions}

On top of the extensive game structure, each player moving at the certain point in the game is endowed with a set of histories that he or she can see from then on.

Consider an extensive game $\mathcal{E}=(\mathcal{G},\{u_i\}_{i\in N})$. A {\bf (short) sight function} for $\mathcal{E}$  \cite{GrossiT12:Sight} is a function $$s:H\backslash Z \rightarrow 2^H\backslash\emptyset$$ associating to each non-terminal history $h$ a finite non-empty and prefix-closed subset of all the histories extending $h$, i.e., histories of the form $(h,h')$. We denote $H\lceil_h = s(h)$ the {\bf sight restriction} on $H$ induced by $s$ at $h$, i.e., the set of histories in player $\mathit{t}(h)$'s sight, and $Z \lceil_h$ their terminal ones. Intuitively, the sight function associates any choice point with those histories that the player playing at that choice point {\em actively} explores.

 In \cite{GrossiT12:Sight} the problem of evaluating intermediate positions is resolved by assuming the existence of an arbitrary preference relation over these nodes, which is common knowledge among the players. What we do instead is to introduce an extension of sight functions that models the evaluation obtained by a concrete search procedure. 
The idea is that in order to evaluate intermediate positions, each player carries out a selection and a random exploration of their continuations, all the way up to the terminal nodes. The information obtained is used as an estimate of the value of those positions. This is an encoding of a basic Monte-Carlo Tree Search \cite{Browne2012}.

Let $(\mathcal{E},s)$ be a tuple made by an extensive game $\mathcal{E}$ and a sight function $s$. Sight function $s^{*}$ is called a {\bf forked extension} of sight function $s$ if the following holds:

\begin{itemize}

\item $s(h) \subseteq s^{*}(h)$ i.e., the forked extension prolongs histories in the sight it extends;

\item For $\lceil^{*}$ being the sight restriction calculated using $s^{*}$ as sight function, we have that: if $h \in s^{*}(h) \setminus Z$ then there exists $h' \in s^{*}(h)$ such that $h \lhd h'$, i.e., $s^{*}(h)$ is made of histories that go all the way up to the terminal nodes. \footnote{A further natural constraint on forked sight functions is that of monotonicity, i.e., players do not forget what they have calculated in the past. Formally $s$ is {\em monotonic} if, for each $h,h'$ such that $\mathit{t}(h)=\mathit{t}({h'})$ and $h\lhd h'$, we have that $s^{*}(h)|_{h'} \subseteq s^{*}(h')$. Albeit natural, this assumption is not needed to prove our results. }

\end{itemize}

%%%%% there is a lot of work to do here, but the idea is really to derive preferences over terminal nodes from the forked restrictions, having a specific operation in mind, say averaging expected utility. it's still not clear how to present the work.however i think it's worth doing it.

 A {\bf Monte-Carlo Tree Game} (MTG) is a tuple $S=(\mathcal{E}, s,s^*)$ where $\mathcal{E}=(\mathcal{G},\{u_i\}_{i \in N})$ is an extensive game, $s$ a sight function for $\mathcal{E}$ and $s^{*}$ a forked extension of $s$. We denote $S\lceil_h=(\mathcal{G}\lceil{_h}, \{u_i\lceil{_h}\}_{i\in N})$ the sight restriction of $S$ induced by $s$ at $h$, where $\mathcal{G}\lceil{_h}$ is the game form $\mathcal{G}$ restricted to $H\lceil_h$ and the utility function $u\lceil{_h}: N \times Z \lceil{_h} \to \mathbb{R}$ is constructed as follows. For each $i\in N, g \in Z \lceil{_h},$ we have: $$u_i\lceil{_h}(g) = \avg_{z \in Z\lceil^{*}_h, g \lhd z } u_i(z)$$

So the utility function at terminal histories in a sight is computed by taking the average\footnote{Averaging has the sole purpose of simplifying notation and analysis, which carries over to {\em any} aggregator, with or without lotteries. Besides, it comes along with a few desirable properties, notably the fact that forked extensions {\em never miss dominated continuations}, i.e., moves that ensure a gain no matter what the opponents do. For quantified restrictions on aggregators cfr. for instance \cite{johan-towards}.} of the histories contained in its forked extension. Notice the following important point: histories in the forked extension are truly treated as "random" explorations, with no rationality assumptions whatsoever, in order to construct a preference relation over $Z \lceil{_h}$. Sight-restriction is applied to players, turn function, strategies and outcome function in the obvious way. Summing up, each structure  $(\mathcal{G}\lceil{_h}, \{u_i\lceil{_h}\}_{i\in N})$ is an extensive game, intuitively the part of the game that the player moving at $h$ is able to see, where the terminal histories are evaluated with a monte-carlo heuristic.

The solution concept proposed in \cite{GrossiT12:Sight} to analyse GSSs is \emph{sight-compatible backwards induction}: a choice of strategy, one per player, that is consistent with the subgame perfect equilibrium of each sight-restricted game. We can encode it as follows.

\begin{definition}{\rm (Sight-compatible BI) Let $S$ be a MTG. A strategy profile $\sigma$ is a \emph{sight-compatible backwards induction} if at each $h\in H$, there exists a terminal history $z\in Z\lceil_{h}$ such that $h\sigma(h)\unlhd z$ and $z\in \textbf{BI}(S\lceil_{h})$. The set of sight-compatible backwards induction outcomes of $S$ is denoted $\textbf{SCBI}(S)\subseteq Z$.}
\end{definition}

Thus, a sight compatible backwards induction is a strategy profile $\sigma$  that, at each history $h$, recommends an action $a$ that is among the actions initiating a backwards induction history within the sight of the player moving at $h$. This, notice, is different from the backwards induction solution of the whole game, because players evaluation of intermediate nodes might not be a {\em correct assessment} of the real outcomes of the game. Grossi and Turrini show that the SCBI solution always exists, even in infinite games.

%GSSs are more realistic representations of  procedural aspects of play than the standard game-theoretic view.
Despite their effort in modelling more {\em procedural} aspects of game play, though, GSSs still lack non-trivial opponent modelling, i.e., players allowing for their opponents to ``miss'' future game developments and evaluate game positions differently (or any higher-order iteration of this belief), while adjusting their behaviour accordingly.

The rest of the paper is devoted to extending MTGs with more realistic but highly more complex reasoning patterns, generalising both GSSs and SCBI. This, it turns out, does not prevent us from having  appropriate well-behaved solution concepts which generalise classical ones, such as backwards induction.

%However, as argued , GSSs suffer from major drawbacks.

%\begin{itemize}
%
%\item Each player implicitly assumes that their sight is consistent with the sight of their opponents, i.e., at future reachable points the opponents will be able to see exactly what the player can see at his choice point.
%
%\item Each player  implicitly assumes their evaluation criteria are consistent with the evaluation criteria of their opponents., i.e., at future reachable points the opponents will be able evaluate game positions in exactly the same way the player does at his choice point.
%
%\item These implicit assumptions are iterated throughtout a players' sight, i.e., the player assumes that their opponents assume that the players moving next will also have consistent sight and evaluation criteria, and so forth.
%
%
%\end{itemize}

%%%%%%%%%%%%%%%%%%%%%%%%%%%%%%%%Section 3%%%%%%%%%%%%%%%%%%%%%%%%%%%%%%%%%%%%%%%%%%%%%%%%%%%%%%%%%%%
\section{Rational beliefs and limited foresight}\label{sec:OM}
%
%Our idea:
%
%$\bullet$ Beliefs in other players' sight, nested sights.
%
%$\bullet$ Beliefs in other players' preference, properties.
%
%$\bullet$ New solution concept, example.
%
%$\bullet$ Relation with SCBI and BI.
%
%$\bullet$ Existence Theorem.
%
%$\bullet$ Algorithm.
%
%\vspace{2ex}

We now introduce an extension of MTGs, where players are allowed for the possibility of higher-order opponent-modelling, i.e., to have an explicit belief about what other players can see and how they evaluate it, a belief about what other players believe other players can see and how they evaluate and so forth, {\em compatibly with players' sight}. We study a solution concept for these games and relate it to known ones from the literature. %generalising the  ones in the literature, and show an existence result under a constructive procedure.

\subsection{Players' sights and belief chains}

Let us introduce the idea behind higher-order opponent modelling in MTGs using an example. We will then move on to define the notions formally.

\begin{example}[An intuitive solution]\label{exa:sight-compatible-beliefs} Consider the game shown in Figure 1. Three players, $Ann$, $Bob$ and $Charles$, move at histories marked $A$, $B$ and $C$, respectively. The circle surrounding history $A$ indicates what $Ann$ believes she can see from history $A$, which we write $b(A)$. This, intuitively, coincides what $Ann$ can {\em actually} see, i.e., it equals $s(A)$, $Ann$'s sight at history $A$. What should $Ann$ do in this situation? This  depends on what $Ann$ believes will happen next. If  $Ann$ knew this,  her choice would only be a maximization problem: finding the action that, given what will happen in the future,  gets her the maximal outcome, according to her evaluation  from $A$ - which we write $\succsim^{A}_{Ann}$. %Again, intuitively, Ann will then ask herself: what will Bob and Charles do?
To find out what Charles will do, Ann considers her belief about what Charles can see from $C$, which we indicate with $b(A)b(C)$. Note this may have nothing to do with what Charles actually sees from C, i.e., b(C). In Figure 1, for instance, Ann believes that Charles can only see $d$ from $C$. The question of what Charles will do is then easily answered, even without considering his preference relation $\succsim^{AC}_{Ann}$, what Ann believes Charles wants from history $C$. Charles, according to Ann, will certainly go to $d$. The next question is: what will Bob do? This, again, will depend on $b(A)b(B)$, the portion of Ann's sight that Ann believes Bob can see from $B$ and on $\succsim^{AB}_{Ann}$, the preferences Ann believes Bob has at $B$. But, at least according to Ann, Bob can also see that Charles can make moves. So, for Bob to decide what to do, he must  first  find out what Charles will do - b(A)b(B)b(C) - according to $\succsim^{ABC}_{Ann}$. This is also an easy task, since $e$ is the only option. The choice at b(A)b(B)b(C) is then determined, but so is then the choice at $b(A)b(B)$. Now all that is left for Ann to do is to solve her maximization problem, determining the choice at $b(A)$.

\end{example}

\begin{figure}
  \begin{center}
  \includegraphics[width=230pt]{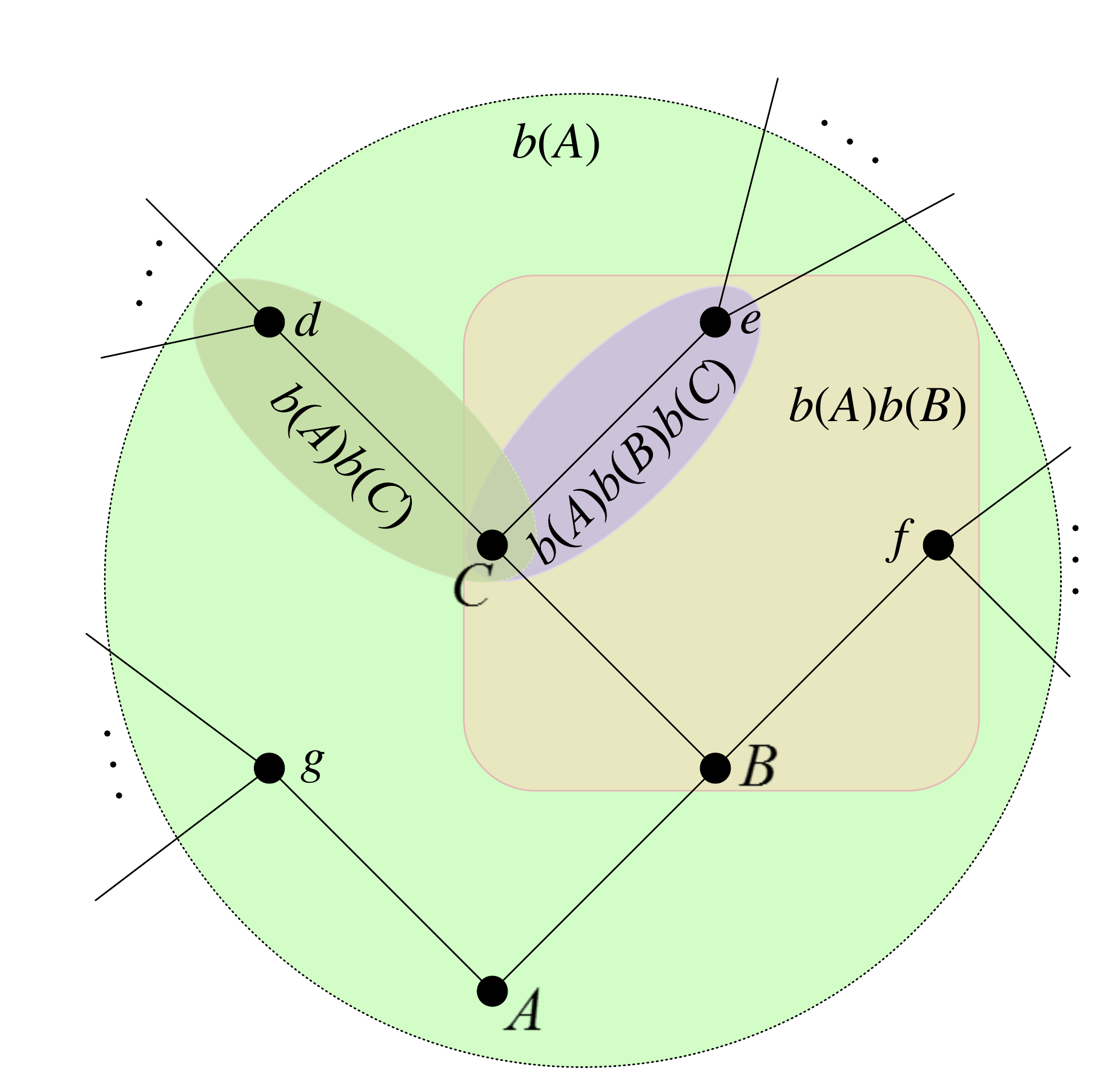}
\label{fig:belief-structure}
  {\caption{A belief structure (modulo evaluations)}}
  \end{center}
\end{figure}

Now we  concentrate on turning the intuitions in the example into formal definitions.
To do so, we introduce the notion of {\em history-sequence}. A history-sequence is a formal device that allows to represent higher-order beliefs about other opponents, consistently with a players' sight.

\begin{definition}[History-Sequences] Consider a MTG $S = (((N,H,t,{\rm\Sigma}_i,o),\{u_i\}_{i\in N}),s,s^{*})$. A \emph{history-sequence} $\textbf{q}$ of $S$ is a sequence of histories of the form $(h_0,h_1,h_2,\cdots,h_k)$ such that

%$\bullet$ $k\in \mathbb{N}$ and $k<\omega$, i.e., $k$ is a natural number and $\textbf{q}$ is finite; this condition is implied by the others

%$\bullet$ $h^0\in H\backslash Z$, i.e., $h^0$ is a non-terminal history in $S$; % this condition is implied by the others

\begin{itemize}

%\item $h^0,h^1,h^2,\cdots,h^k \in H$
\item $h_j\in H\lceil_{h_0} $ for every $j\in \{1,2,\cdots k\}$, i.e., histories following $h_0$ in the sequence are  histories within the sight of the player moving at $h_0$; % this change can have ripercussions on the technical details. It should work well and avoid complications, but some details might need change

\item $h_j\lhd h_{j+1}$ for each $j$ with $0\leq j< k$, i.e., each history is a strict postfix of the ones with lower index;

\end{itemize}
%$\bullet$ $k< l(s(h^0))$, where $l(s(h^0))$ is the length of the longest histories in $s(h^0)$.
%}

\end{definition}

The underlying idea behind this definition is to consider the higher-order point of view of the player moving at $h_0$. Expressions of the form $(h_0,h_1,h_2,\cdots,h_k)$  encode the belief that player moving at $h_0$ holds about the belief that player moving at $h_{1}$ holds about the belief that player moving at $h_{2}$ holds \ldots about what the player moving at $h_{k}$ can see and what the evaluation is of the corresponding terminal histories. We use ${\bf Q}$ to denote the set of history-sequences of $S$.

\vspace{1ex}

Building upon the notion of history-sequence, we can define  what we call {\em sight-compatible} belief structures, associating each history-sequence with a set of histories and an evaluation over the terminal ones in this set.

\begin{definition}{\rm (Sight-compatible belief structures) Let $S$ be a MTG.  A \emph{sight-compatible belief structure} ${\bf B}$ for $S$ is a tuple $({\bf B}_H,{\bf B}_P)$ such that ${\bf B}_H$ is a function ${\bf B}_H: {\bf Q}\rightarrow 2^H$, associating to each history-sequence $(h_0,h_1,h_2,\cdots,h_k)$ a set of histories in $s(h_0)$ extending $h_k$, and ${\bf B}_P$ is a function ${\bf B}_P: {\bf Q}\rightarrow 2^Z$ associating to each history-sequence ${\bf q}$  a set of terminal histories extending histories in ${\bf B}_H({\bf q})$.  ${\bf B}$ satisfies the following conditions:

\begin{itemize}

\item  (Corr) $\forall \textbf{q}\in {\bf Q}$ with ${\bf q}=(h_0)$,  then ${\bf B}_H(\textbf{q}) = H\lceil_{h_0}$ whenever $t(h_k)$ the belief of a player about what he himself can see is correct. \footnote{One might want to impose stronger variants of correctness. For instance the fact that if a player can see he will be moving again, then he will consider at least as much as he is considering now, from that history on: $\forall \textbf{q}\in {\bf Q}, \mbox{ if } q=(h_0, h_1, \ldots, h_k) $, $q^{\prime}=(h_1, \ldots, h_k) $ and $t(h_0)=t(h_1)$ then ${\bf B}_H(\textbf{q}') = {\bf B}_H(\textbf{q})|_{h_0}$.

We can also impose that this fact is common knowledge among the players: $\forall \textbf{q}\in {\bf Q}, \mbox{ if } q=(h_0, h_1, h_{i-1}, h_i, h_{i+1} \ldots, h_k) $, $q^{\prime}=(h_0, h_1, h_{i-1}, h_{i+1} \ldots, h_k) $ and $t(h_i)=t(h_{i-1})$ then ${\bf B}_H(\textbf{q}') = {\bf B}_H(\textbf{q})|_{h_i}$. }

%=\item $\forall \textbf{q}\in Q$ with $\textbf{q}=(h^0,h^1,h^2,\cdots,h^k)$,  we have that ${\bf B}_H(\textbf{q}) \subseteq H|_{s(h_0)} \setminus\emptyset$; this should follow from the next condition

\item (Mon of ${\bf B}_H$) $\forall$ $\textbf{q}$, $\textbf{q}'\in {\bf Q}$, if $\exists h'\in {\bf B}_H(\textbf{q})$ s.t., $\textbf{q}'=(\textbf{q},h')$, then ${\bf B}_H(\textbf{q}')\subseteq {\bf B}_H(\textbf{q})|_{h'}$, i.e., if a player believes someone is able to perceive a portion of the game, then he is able to perceive that portion himself.

\item (Mon of ${\bf B}_P$) $\forall$ $\textbf{q}$, $\textbf{q}'\in {\bf Q}$, if $\exists h'\in {\bf B}_H(\textbf{q})$ s.t., $\textbf{q}'=(\textbf{q},h')$, then ${\bf B}_P(\textbf{q}')\subseteq {\bf B}_P(\textbf{q})|_{h'}$, i.e., if a player believes someone is able to explore a position, then he is able to perceive that exploration himself.

\end{itemize}
%$\bullet$ For every history $h'$, that is non-terminal and non-empty, $s(h)s(h')$  is a finite tree, and the nodes in it is a subset of $s(h)$.

%$\bullet$ For every $h''\in s(h)s(h')$, that is non-terminal and non-empty in $s(h)s(h')$,

%$~~~~\diamond$ $s(h)s(h')s(h'')$ is a finite tree, and nodes in it is a subset of  $s(h)s(h')$.

%$~~~~\diamond$ And for every $h'''\in s(h)s(h')s(h'')$ that is non-terminal and non-empty in $s(h)s(h')s(h'')$, s.t. $h''\lhd h'''$, $s(h)s(h')s(h'')s(h''')$ is a finite tree, and the nodes in it is a subset of $s(h)s(h')s(h'')$.

%$~~~\diamond$ Such nesting continues until reaching the history with empty nested belief.

For  ${\bf q}=(h_0,h_1,h_2,\cdots,h_k)$, ${\bf B}_P({\bf q})$ denotes the higher-order beliefs (in the order given by {\bf q}) about how player moving at $h_k$ is evaluating the terminal histories in  ${\bf B}_H({\bf q})$ under the unique forked extension of $s$ whose terminal histories are ${\bf B}_P({\bf q})$. $ \succeq^{{\bf B}_P(\textbf{q})}_{i}$ denotes the induced preference relation, one per player.
}

\end{definition}

The conditions above, we argue, are most natural constraints on sight-compatible higher-order beliefs. For the time being we do not commit ourselves to any other constraints on either  ${\bf B}_P$ or  ${\bf B}_H$, but we acknowledge that different contexts may warrant further constraints on both.

\begin{definition}[Epistemic Monte-Carlo Tree Games]  An \emph{Epistemic Monte-Carlo Tree Games} (EMTGs) is a tuple $\mathcal{S}=(S,{\bf B})$ where $S$ is a MTG and ${\bf B}$ a sight-compatible belief structure for $S$.
\end{definition}

An EMTG is obtained by assigning a sight-compatible belief structure to a MTG. One should observe how sight-compatible belief structures induce, at each history, a whole collection of extensive games, one for each possible history-sequence. For instance, the one resulting from Ann's sight and her evaluation, the one resulting from Ann's belief about Bob's sight and his evaluation and so forth. Structures of the form $\mathcal{S}\lceil_{{\bf B}(\textbf{q})}$ can now be naturally defined, as restrictions induced by ${\bf B}(\textbf{q})$ on $\mathcal{S}$, adopting ${\bf B}_H(\textbf{q})$ as sight-restriction, and ${\bf B}_P(\textbf{q})$ as evaluation function, with the induced preference relation. 

\subsection{Analysing EMTGs}

Example \ref{exa:sight-compatible-beliefs} has illustrated a natural notion of solution in an epistemic MTG, where each player calculates a best action in his or her sight restriction according to his or her evaluation criteria, recursively computing both the sight and the evaluation criteria of the other players. This is the idea behind the solution concept we propose for EMTGs.

%\begin{definition}{\rm (Belief-based Sight-Compatible Solution) Let $S=(G,{\bf B})$ be an extensive game with Belief-based Sight. A terminal history $z^*$ is a Belief-based Sight-Compatible Solution (BSCS) of $S$, $sol(S)$, if for each intermediate history $h\lhd z^*$, $z^*\lceil_h = \emph{best}\{(a, h^*)\}$, where $a$ is an action following $h$ in $s(h)$, and $h^*$ is a path starting from $(ha)$ in $s(h)$, which is the one that $t(h)$ believes the following players will choose.}
%\end{definition}
%In other words, $h^*$ is a sequence of initial moves of solutions of the corresponding associated trees at a further level of nesting.

%%begin[chanjuan-NOV.29]

%
\begin{definition}[Nested Beliefs Solution]\label{def:NBS} Let $\mathcal{S}=(S,{\bf B})$ be an EMTG and let $\textbf{q}=(h_0, h_1,\cdots, h_k)$ be a history-sequence.  A strategy profile $\sigma\lceil_{{\bf B}(\textbf{q})}$ is a \emph{Nested Beliefs Solution} (NBS) of $\mathcal{S}\lceil_{{\bf B}(\textbf{q})}$ if:
\begin{description}

\item[Base step] For each $h'\in H\lceil_{{\bf B}_H(\textbf{q})}$ that is a quasi-terminal history of $H\lceil_{{\bf B}_H(\textbf{q})}$, we have that $h'\sigma\lceil_{{\bf B}_H(\textbf{q},h')}(h') \succeq^{{\bf B}_P(\textbf{q},h')}_{t(h')} h',\sigma'\lceil_{{\bf B}_H(\textbf{q},h')}(h')$ for any $\sigma'\lceil_{{\bf B}_H(\textbf{q},h')}$ that agrees with $\sigma\lceil_{{\bf B}(\textbf{q})}$  up to $h'$.

\item[Induction step] For each $h'\in H\lceil_{{\bf B}_H(\textbf{q})}$ that is neither terminal nor quasi-terminal in $H\lceil_{{\bf B}_H(\textbf{q})}$,  we have that

\begin{itemize}

\item $\sigma\lceil_{{\bf B}(\textbf{q})}(h')$ agrees at $h'$ with some Nested Beliefs Solution of  $\mathcal{S}\lceil_{{\bf B}(\textbf{q},h')}$.

\item If, for each $h'\in H\lceil_{{\bf B}_H(\textbf{q})}$ that is neither terminal nor quasi-terminal in $H\lceil_{{\bf B}_H(\textbf{q})}$, we have that $\sigma'\lceil_{{\bf B}(\textbf{q})}(h')$ agrees at $h'$ with some Nested Beliefs Solution of  $\mathcal{S}\lceil_{{\bf B}(\textbf{q},h')}$ then the outcome $z'$ generated by $\sigma'\lceil_{{\bf B}(\textbf{q})}$ following $h$ and the outcome $z$ generated by $\sigma\lceil_{{\bf B}(\textbf{q})}$ following $h$ are such that $z \succeq^{{\bf B}_P(\textbf{q})}_{t(h_k)} z'$.

\end{itemize}
\end{description}
\end{definition}

We denote $ \textbf{NBS}(\mathcal{S}\lceil_{{\bf B}(\textbf{q})})$  the set of NBS outcomes of $(\mathcal{S},\textbf{q})$. The composition of such outcomes yields our game solution.

Intuitively, a Nested Beliefs Solution of some game $\mathcal{S}\lceil_{{\bf B}(\textbf{q})}$ is a best response to all Nested Belief Solutions at deeper level, e.g., of each  $\mathcal{S}\lceil_{{\bf B}(\textbf{q},h')}$. Notice that because of the properties of sight functions the depth iteration is bound to reach a fixpoint.

\begin{example}

Let's go back to Figure 1 and compute the $\textbf{NBS}$ at history $A$. We know there are four relevant histories sequences: $(A)$, $(A,C)$, $(A,B)$ $(A,B,C)$. To each of them we can associate the corresponding beliefs, as follows:

\begin{itemize}

\item $H\lceil_{{\bf B}_H(\textbf{A})} = \{g, d, e, f, C, B, A\} = H\lceil_{h_0} $

\item $H\lceil_{{\bf B}_H(\textbf{A,B})} = \{e, C, f\} $

\item $H\lceil_{{\bf B}_H(\textbf{A,C})} = \{d\} $

\item $H\lceil_{{\bf B}_H(\textbf{A,B,C})} = \{e\}$

\end{itemize}

Let us know, for each histories sequence $q$ specify the preference relation $ \succeq^{{\bf B}_P(\textbf{q})}_{t_k}$ (modulo reflexivity and transitivity), which is all we need to compute $\textbf{NBS}$. 

\begin{itemize}

\item $ \succeq^{{\bf B}_P(\textbf{A})}_{Ann} = \{(d,g), (g,e), (e,f), (f,e)\} $

\item $ \succeq^{{\bf B}_P(\textbf{A,B})}_{Bob} = \{(e,f), (f,d)\}$

\item $ \succeq^{{\bf B}_P(\textbf{A,C})}_{Charles} = \{\}$

\item $ \succeq^{{\bf B}_P(\textbf{A,B,C})}_{Charles} = \{\}$

\end{itemize}

Consider now the following strategy $\sigma\lceil_{{\bf B}(\textbf{A})}$ \footnote{Slightly abusing notation, but unambiguosly, we identify actions chosen by the strategy with the resulting histories.}

\begin{itemize}

\item $\sigma\lceil_{{\bf B}(\textbf{A})}(A)_{Ann}=g$ 

\item $\sigma\lceil_{{\bf B}(\textbf{A})}(B)_{Bob}=C$ 

\item $\sigma\lceil_{{\bf B}(\textbf{A})}(C)_{Charles}=d$ 

\end{itemize}

Is $\sigma$ a Nested Beliefs Solution of $\mathcal{S}\lceil_{{\bf B}(\textbf{A})}$?

The condition at the base step is met by $\sigma\lceil_{{\bf B}(\textbf{A})}(C)_{Charles}=d$. 

Lets now look at  $\sigma\lceil_{{\bf B}({A})}(B)_{Bob}=C$. Is $C$ compatible with the best Nested Beliefs Solution of $\mathcal{S}\lceil_{{\bf B}({A,B})}$? 
We need first to compute all NBS of $\mathcal{S}\lceil_{{\bf B}({A,B})}$. Luckily there are not so many.
Every such strategy must be of the form $\sigma'\lceil_{{\bf B}(\textbf{A,B})}(C)_{Charles}=e$ and be the best among the strategies agreeing with NBS of $\mathcal{S}\lceil_{{\bf B}(\textbf{A,B,C})}$ at $C$. So, given the preferences of $B$, be such that $\sigma'\lceil_{{\bf B}(\textbf{A,B})}(B)_{Bob}=C$. This is indeed what $\sigma$ does. 

However notice that given the preference of $A$, $\sigma$ is {\bf not} behaving as a NBS at $A$, because $Ann$ prefers $d$ to $g$.

The strategy $\sigma^{*}\lceil_{{\bf B}({A})}$ only disagreeing with $\sigma\lceil_{{\bf B}({A})}$ at $(A)$, and being such that $\sigma\lceil_{{\bf B}({A})}(A)_{Ann}=B$, is a NBS of $\mathcal{S}\lceil_{{\bf B}(A)}$.

\end{example}

The composition of Nested Beliefs Solutions constitutes a rational outcome of the game.

\begin{definition}[Sight-Compatible Epistemic Solution] Let $\mathcal{S}=(S,{\bf B})$ be an EMTG.  A strategy profile $\sigma$ is a \emph{Sight-Compatible Epistemic Solution} (SCES) if at each $h\in H\setminus Z$, there exists a terminal history $z\in Z\lceil_{{\bf B}_H(h)}$ such that $h\sigma(h)\unlhd z$ and $z\in \textbf{NBS}(\mathcal{S}\lceil_{{\bf B}(h)})$.
\end{definition}

We denote $\mathbf{SCES}$ the set of Sight-Compatible Epistemic Solutions of $\mathcal{S}$.

%%end[chanjuan-NOV.29]

A SCES is the composition of best moves of players at each history. Each such move is a best response to what the current player believes other players will do and this belief is {supported} by all higher-order beliefs, compatible with the player's sight, about what the opponents can perceive and how they will evaluate it.

\subsubsection{Computing rational solutions}

Algorithm $Sol(\mathcal{S})$ below takes as input an EMTG and returns a path obtained by composing locally rational moves, compatible with players' higher-order beliefs about sights and evaluation criteria of their opponents. Algorithms \ref{sol(s)} calls Algorithm \ref{BSBI}, which in turn calls Algorithms \ref{NBS} and \ref{compose}. For technical convenience, we define $\emph{VLP}$ to be a dummy always dominated history.

\iffalse
\begin{algorithm}[htbp]
\caption{Solution of $\mathcal{S}$}\label{sol(s)} {\em Sol}$(\mathcal{S})$\\
\KwIn{ An EMTG $\mathcal{S}=(S,{\bf B})$}
\KwOut{A solution of $\mathcal{S}$, \emph{Path}}
\Begin{
       $h\leftarrow \varepsilon$;
       $\emph{Path}\leftarrow \varepsilon$;\\
       \While{$h\notin Z$}{
       $a\leftarrow\emph{ BSBI}(\mathcal{S},h)$;  \tcc*[f]{\footnotesize The best move at $h$}\\
       $\emph{Path}\leftarrow (\emph{Path},a)$;\\
       $h\leftarrow (ha)$;\\
       }
       Return \emph{Path};

}
\end{algorithm}
\fi

\begin{algorithm}[h]
\caption{Solution of $\mathcal{S}$}\label{sol(s)} $Sol(\mathcal{S}$)\\
\KwIn{ An EMTG $\mathcal{S}=(S, {\bf B})$}
\KwOut{A terminal history $h$ of $\mathcal{S}$}
\Begin{
       $h\leftarrow \varepsilon$;\\
       \While{$h\notin Z$}{
       %$\sigma_{t(h)}\leftarrow\emph{BSBI}(\mathcal{S},h)$;  \tcc*[f]{Best move at $h$}\\
       $h\leftarrow (h,\emph{BSBI}(\mathcal{S},h))$; \tcc*[f]{\footnotesize NBS at $h$}\\
       }
       Return $h$;

}
\end{algorithm}

\begin{algorithm}[h]
\caption{The current best move}\label{BSBI} $\emph{BSBI}(\mathcal{S}, h)$\\
\KwIn{ A game $\mathcal{S}=(S, {\bf B})$, and a history $h$}
\KwOut{NBS move $a$ at $h$}
\Begin{
       %\eIf{$h\in Z_{\vv{B}(h)}$}{
       %Return $\varepsilon$;\\
      % }
       {\For {each $h'\in {\bf B}_{H}(h)$ and $h'\neq h$}{
         $\emph{Continuations}[h'] \leftarrow \emph{NBS}(S,(h,h'))$;\tcc*[f]{\footnotesize Store NBS actions in an array, one for each $h'$}\\
       }

       }
       Return $\emph{BB}(\mathcal{S}, (h), Continuations)$;\\
}
\end{algorithm}

\begin{algorithm}[h]
{\caption{Beliefs of moves of following players}\label{NBS}} $\emph{NBS}(\mathcal{S}, {\bf q})$\\
\KwIn{ A game $\mathcal{S}=(S, {\bf B})$, and a history sequence ${\bf q}=(h_0,h_1,h_2,\cdots,h_k)$}
\KwOut{An action $a$ following $h_{k}$}
\Begin{
       \eIf{$h_k\in Z\lceil_{{\bf B}_{H}({\bf q})}$}{
       Return $\varepsilon$;\\
       }{\For {each $h_{k+1}\in {\bf B}_{H}({\bf q})$ and $h_{k+1}\neq h_{k}$}{
         $\emph{Continuations}[h_{k+1}] \leftarrow \emph{NBS}(S,({\bf q},h_{k+1}))$;\tcc*[f]{\footnotesize Store NBS actions in an array, one for each $h_{k+1}$}\\
       }
       Return $\emph{BB}(\mathcal{S}, \textbf{q}, Continuations)$;\\
       }

}
\end{algorithm}

\begin{algorithm}[h]
\caption{Best Branch}\label{compose} {$\emph{BB}(\mathcal{S}, \textbf{q}$, \emph{Continuations})}\tcc*[f]{\footnotesize Compose chosen moves (in array Continuations), thus get all paths following $h_k$ and choose a best move following $h_k$ }\\
\KwIn{ A game $\mathcal{S}$, a history sequence $\textbf{q}$, an array \emph{Continuations}}
\KwOut{A best move following $h_k$ determined by \emph{Continuations}}
\Begin{
       $\emph{bestpath}\leftarrow \emph{VLP}$; \tcc*[f]{\emph{VLP} is a dominated history for all players}\\
       \For (\tcc*[f]{\footnotesize $a$ is any action following $h_k$, next we choose an optimal one in ${\bf B}_{H}(\textbf{q})$}){each $(h_k,a)\in {\bf B}_{H}(\textbf{q})$}{
           $\emph{TP}\leftarrow (h_k,a)$;\\
           \While{Continuations$[$TP$]$ is defined in array Continuations} {
                  $\emph{TP}\leftarrow(\emph{TP}, \emph{Continuations}[\emph{TP}])$;\\}
           %$\emph{Path}[\emph{curr}]\leftarrow \emph{TP}$; \tcc*[f]{The best path following $curr$}\\
           \If {TP $\succ^{B_P(\textbf{q})}_{t(h_k)}$ bestpath}{
              $\emph{bestpath}\leftarrow \emph{TP}$;\\
              $\emph{bestmove}\leftarrow a$;\\
              }
           }
       Return $\emph{bestmove}$;
}\end{algorithm}

\iffalse
\begin{algorithm}
\caption{Compose actions}\label{compose} {\em compose(Continuations, h)}\tcc*[f]{Compose the move (in array Continuations) that was chosen for each decision point(which are successors of $h$), and thus get a path starting from $h$}\\
\KwIn{ An array \emph{Continuations}, containing the action \emph{Continuations}[$h'$] chosen at each history $h'$, and a history $h$}
\KwOut{The path(starting from $h$) determined by \emph{Continuations}}
\Begin{
$\emph{TP}\leftarrow h$;\\
\While{$\emph{Continuations}[h]$ is defined in array \emph{Continuations}} {
$TP\leftarrow(\emph{TP}, Continuations(h))$;\\
$h\leftarrow Continuations(h)$;
}
Return \emph{TP};
}
\end{algorithm}
\fi

The following theorem shows that every EMTG has a Sight-Compatible Epistemic Solution. Its proof consists in constructively building the desired strategy profile.

\begin{theorem}[Existence Theorem]

Let $\mathcal{S}=(S,{\bf B})$ be an EMTG.  There exists a strategy profile $\sigma$ that is a {Sight-Compatible Epistemic Solution} for $\mathcal{S}=(S,{\bf B})$.

\end{theorem}

\begin{proof}

Let $H$ be the set of histories in $\mathcal{S}$. For every $h\in H$ set $\sigma(h):= a$ for $a\unlhd h^{*}$ and $h^{*}$ be the outcome returned by Algorithm \ref{sol(s)} on input $\mathcal{S}\lceil_{{\bf B} (h)}$.  That the Algorithm returns a profile $o(\sigma)$ such that $\sigma$ satisfies the conditions of Definition \ref{def:NBS} at each history is a lengthy but relatively straightforward check, which we omit for space reasons.$\hfill\Box$ \end{proof}

\begin{theorem}[Completeness Theorem]

Let $\mathcal{S}=(S,{\bf B})$ be a finite EMTG and let $\sigma$ be a {Sight-Compatible Epistemic Solution} for $\mathcal{S}=(S,{\bf B})$. There exists an execution of Algorithm \ref{sol(s)} returning $o(\sigma)$.

\end{theorem}

\begin{proof}

Let $\mathcal{S}=(S,{\bf B})$ be a finite EMTG and let $\sigma$ be a {Sight-Compatible Epistemic Solution} for $\mathcal{S}=(S,{\bf B})$. Now choose an execution of Algorithm \ref{sol(s)} that is compatible with the action selection that, at each history sequence, is made by $\sigma$, which exists by construction. The finiteness assumption ensures termination.$\hfill\Box$
\end{proof}

The following observations illustrate the relation between SCES and the other two relevant solution concepts in the literature: SCBI \cite{GrossiT12:Sight} and classical BI \cite{Osborne1994}. They  specify precise  conditions under which our solution concept collapses into these two.

\begin{proposition}\label{prop:SCBI-SCES} Let $\mathcal{S}=(S, {\bf B})$ be an EMTG. If for any history-sequence $\mathbf{q}$ = $(h_0, h_1, \cdots, h_k)$ and any history $h'\in {\bf B}_H(\mathbf{q})$, ${\bf B}_H(\mathbf{q},h')={\bf B}_H(\mathbf{q})|_{h'}$, and ${\bf B}_P(\mathbf{q})=B_P(h_0)$, then $\mathbf{SCES}(\mathcal{S})$=$\mathbf{SCBI}(\mathcal{S})$.
\end{proposition}

So,  if the current player believes the following players' sights and evaluation criteria, together with their beliefs about other players' sights and evaluation criteria, are coherent with his', then SCES is equivalent to SCBI.

We know that the solution concept BI is a special case of SCBI, and therefore also of SCES.

\begin{proposition}\label{prop:BI-SCES} Let $\mathcal{S}=(S, {\bf B})$ be an EMTG. If, for any history-sequence $\mathbf{q}$ = $(h_0, h_1, \cdots, h_k)$,  we have that ${\bf B}_H(\mathbf{q})=H|_{h_k}$ and that ${\bf B}_P(\mathbf{q})=P_{t(h_k)}$ then $\mathbf{SCES}(\mathcal{S})$=$\mathbf{BI}(\mathcal{S})$.
\end{proposition}

The above result says that SCES coincides with standard backwards induction solution if, at each history, we have that higher-order beliefs about sight and evaluation criteria are {\em coherent} with the real subgame the current player faces and the preference relation the current player holds.

%The set of Belief-based Sight-Compatible BI solutions of $S$ is denoted $\textbf{SCES}(S)\subseteq Z$.

Despite the crucial presence of higher-order beliefs about sight-restricted games, we can show the following fairly surprising complexity result.

\begin{proposition}\label{prop:complexity}

Given a  finite EMTG $\mathcal{S}$, the problem of computing a SCES of $\mathcal{S}$ is {\em \P-TIME} complete.
\end{proposition}

\begin{psketch}

For the upper bound, the key fact is that algorithm $Sol(S)$ runs in time $\mathcal{O}((n\log{n})^2)$, with $n$ being the cardinality of the set of histories of $\mathcal{S}$. This follows from the equations and facts below, where $b$ and $d$ are the largest number of branches and the depth of game tree respectively: 1). $T(\emph{sol}(S))$ = $\mathcal{O}(T (\emph{BSBI} * d))$; 2). $T(\emph{NBS})$ = $\mathcal{O}(T(\emph{BSBI}))$; ~~3). $T(\emph{BB})$ = $\mathcal{O}(b*d)$;  ~~4). Let $f(d)$=$T(\emph{BSBI})$, then $f(d)$=$\mathcal{O}(b*f(d-1) + b^2*f(d-2) + b^3*f(d-3)+\cdots+ b^{d-1}f(1)+ T(BB))$=$\mathcal{O}(d*2^{d}*b^{d})$;  ~~5). $d\leq log(n)$ and $2^d\leq b^d\leq n$.
\P-TIME hardness is a consequence of \cite[Theorem 2]{jakub}, which shows that BI is \P-TIME hard, and Proposition \ref{prop:BI-SCES}.%, showing that BI is a special case of SCES.
\end{psketch}

As a side remark,  using a similar argument and Proposition \ref{prop:BI-SCES} we are able to show that computing {\sc SCBI} solutions is  \P-TIME complete.

\section{Conclusions and potential developments}\label{sec:Conclusion}

We have proposed a model for decision-making among resource-bounded players in extensive games,  integrating an analytical perspective coming game theory with a procedural perspective coming from AI. In particular we have studied players with limited foresight which can reason about their opponents, constructing beliefs about their limited abilities for calculation and evaluation, showing that our novel games have a well-behaved solution, generalising existing ones in the literature.

%An interesting computational addition to our framework is to {\em dynamify} our static sight structure using Monte-Carlo simulations \cite{AbramsonK87} that randomly expand the frontier of sight towards terminal nodes. Priority-based preferences can be updated accordingly. Another important practical connection, given our initial motivation, would link up with the use of {\em opponent modelling} in automated game solving \cite{schadd2007opponent}.

There are interesting modelling issues, as noted previously. Our game models strike a balance between simple trees as used for BI and more complex models as found in epistemic game theory \cite{perea}. Here, what we left open is the relation between EMTGs and the Extensive Games with Awareness of \cite{HalpernR06}. We expect that the correspondence for GSSs of Theorem 3 in \cite{GrossiT12:Sight} can be lifted to EMTGs, using an iteration of the awareness functions $Aw_i$ for  players $i$ to simulate the believed game at a history sequence.  We stress, though, that the specific features of EMTGs give them an independent conceptual and technical interest. The  emphasis on limited foresight (as opposed to  perceiving a novel extensive game in \cite{HalpernR06}) makes them a natural candidate for addressing Rubinstein's modelling challenge \cite{Rubinstein-Bounded}, while still supporting an efficient algorithm to calculate the game equilibria.

Finally, our analysis raises several issues of logical definability and styles of reasoning.  We believe that our solution concept is still definable in a computationally well-behaved logical language, a natural candidate being the fixed-point logic FOL(FP), shown in \cite{johan} to express backwards induction. What is new in our setting is that the reasoning underpinning our main theorems is a mixture of a backward induction style with a forward induction style \cite{perea,johan-book}, since we have to evaluate what players further down in the game tree are going to do according to players whose moves {\em occurred} earlier on in the game. %This makes our framework a good candidate for developing theory based on concrete intuitive scenarios at this still tentative interface of logic and game theory.

\section{Acknowledgements} The author acknowledges the support of Imperial College London for the Junior Research Fellowship "Designing negotiation spaces for collective decision-making" (DoC- AI1048). He is besides extremely thankful to Johan van Benthem, Davide Grossi and Chanjuan Liu for their feedback on earlier versions of the paper.

\newpage
%\newpage
\bibliographystyle{named}
{\bibliography{oppmodelling}}

\end{document}